\documentclass{article}

 \usepackage[preprint]{neurips_2026}

\usepackage[utf8]{inputenc} 
\usepackage[T1]{fontenc}    
\usepackage[colorlinks = true,
            linkcolor = blue,
            urlcolor  = blue,
            citecolor = blue,
            anchorcolor = blue]{hyperref}      
\usepackage{url}            
\usepackage{booktabs}       
\usepackage{amsfonts}       
\usepackage{nicefrac}       
\usepackage{microtype}      
\usepackage{xcolor}         

\setcitestyle{authoryear}

\usepackage{amsmath,amsfonts,bm}

\def\eqref#1{equation~\ref{#1}}

\def\1{\bm{1}}

\DeclareMathAlphabet{\mathsfit}{\encodingdefault}{\sfdefault}{m}{sl}
\SetMathAlphabet{\mathsfit}{bold}{\encodingdefault}{\sfdefault}{bx}{n}

\newcommand{\Var}{\mathrm{Var}}

\newcommand{\Cov}{\mathrm{Cov}}

\usepackage{amsthm}
\usepackage{multirow}
\usepackage{graphicx}
\usepackage[most]{tcolorbox}
\usepackage{color, colortbl}
\usepackage{subcaption}
\theoremstyle{plain}
\newtheorem{proposition}{Proposition}
\definecolor{lightgray}{gray}{0.95}
\definecolor{lightblue}{RGB}{230,240,255}

\title{DVAO: Dynamic Variance-adaptive Advantage Optimization for Multi-reward Reinforcement Learning}

\author{%
  Guochao Jiang\thanks{Corresponding author.}, Jingyi Song, Guofeng Quan, Chuzhan Hao, Guohua Liu, Yuewei Zhang \\
Alibaba Cloud Computing \\
\texttt{anyue.jgc@alibaba-inc.com}
}

\begin{document}

\maketitle

\begin{abstract}
  Reinforcement Learning has become a standard paradigm for aligning Large Language Models with human intent and task requirements. While Group Relative Policy Optimization offers an efficient, value-model-free alternative to Proximal Policy Optimization, adapting it to real-world multi-reward settings remains challenging. Standard scalarization practices, such as Reward Combination and Advantage Combination, suffer from significant drawbacks: Reward Combination frequently generates advantages with excessively large squared magnitudes that lead to training instability, while Advantage Combination relies on static hyperparameters and ignores cross-objective correlations. To address these limitations, we propose \textbf{D}ynamic \textbf{V}ariance-adaptive \textbf{A}dvantage \textbf{O}ptimization (\textbf{DVAO}), which dynamically adjusts combination weights based on the empirical reward variance of each objective within a rollout group, effectively up-weighting objectives with a stronger learning signal while suppressing noisy ones. We mathematically prove that DVAO maintains bounded advantage magnitudes for stable training and introduces a self-adaptive cross-objective regularization mechanism. Extensive experiments on mathematical reasoning and tool-use benchmarks using Qwen3 and Qwen2.5 models demonstrate that DVAO significantly outperforms baseline methods, achieving a superior multi-objective Pareto frontier and robust training stability.
\end{abstract}

\section{Introduction}
Large Language Models (LLMs) have demonstrated remarkable capabilities across a wide range of natural language processing tasks \citep{llm-survey}, including Qwen3 \citep{qwen3}, Kimi K2.5 \citep{k2.5}, and DeepSeek-R1 \citep{deepseek-r1}. To align these models with human intent and specific task requirements, Reinforcement Learning (RL) has become a standard paradigm \citep{rl-survey, sft-rl}. Recently, Group Relative Policy Optimization (GRPO) \citep{grpo} and its variants \citep{dapo, gspo, vcrl} have emerged as highly efficient alternatives to Proximal Policy Optimization (PPO) \citep{ppo} for LLMs. By eliminating the need for a separate value model and relying instead on relative advantage estimation within a sampled group of rollouts, GRPO significantly reduces memory overhead and simplifies the training pipeline \citep{liu2025part}.

However, deploying LLMs in real-world scenarios rarely involves optimizing a single, isolated metric. Practical applications dictate multi-objective requirements: a model must not only provide accurate answers but also adhere to length constraints \citep{efficient_survey1, efficient_survey2}, minimize bug rates in code generation \citep{DBLP:journals/ese/TambonDNKDA25, gao2025survey}, maintain a low hallucination rate \citep{hallucination_survey1, hallucination_survey2}, and keep correct tool-calling format in tool-use \citep{search_r1, airrag}. Adapting GRPO to this multi-reward setting is non-trivial. The standard practice involves scalarization-either linearly combining the raw rewards (Reward Combination) or independently normalizing the rewards and then combining their respective advantages (Advantage Combination).

Despite their widespread use, both methods suffer from significant theoretical and practical drawbacks. As we demonstrate in this work, the Reward Combination method frequently generates advantages with excessively large squared magnitudes than the Advantage Combination method, which translates to erratic policy gradients and training instability. Conversely, while the Advantage Combination method normalizes these magnitudes, it relies on static hyperparameters and completely isolates the objectives during normalization. This naive decoupling fails to capture the intricate correlations—whether synergistic or antagonistic—between different objectives during a single rollout, often leading to suboptimal trade-offs.

To address these fundamental limitations, we propose \textbf{D}ynamic \textbf{V}ariance-adaptive \textbf{A}dvantage \textbf{O}ptimization (\textbf{DVAO}). DVAO elegantly bridges the gap between stability and objective synergy by dynamically adjusting the combination weights based on the empirical reward variance of each objective within the rollout group. This completely data-driven method up-weights objectives with higher variance—indicating a stronger learning signal—while suppressing noisy, low-variance objectives. Crucially, we mathematically prove that DVAO not only bounds the advantage magnitude for stable training but also introduces a self-adaptive cross-objective regularization mechanism. In DVAO, the gradient contribution of a single objective is modulated by the overall multi-objective performance of that specific rollout, ensuring a holistic optimization trajectory.


In summary, we theoretically expose the fundamental flaws of existing scalarization methods in multi-reward GRPO—namely magnitude explosion and objective isolation—and propose Dynamic Variance-adaptive Advantage Optimization to address these limitations. DVAO is a fully dynamic, hyperparameter-free weighting scheme that we mathematically prove maintains bounded advantage magnitudes while introducing an implicit cross-objective regularization mechanism to promote synergistic learning. Extensive empirical evaluations on mathematical reasoning and tool-use benchmarks demonstrate that DVAO significantly outperforms baseline methods, accelerating convergence and consistently achieving a superior multi-objective Pareto frontier without sacrificing robust training stability.
\section{Preliminaries}
Recently, GRPO \citep{grpo} and its variants, including Dynamic Sampling Policy Optimization (DAPO) \citep{dapo} and Group Sequence Policy Optimization (GSPO) \citep{gspo}, have become widely used algorithms for policy optimization due to their simplicity and efficiency. Unlike Proximal Policy Optimization (PPO) \citep{ppo}, GRPO gains more flexibility by eliminating the value model and using relative advantage within group.

GRPO initially calculates the relative advantage for a single reward and then performs policy optimization. However, real-world tasks often have multi-objective requirements. In addition to the accuracy of the task itself, there may be other requirements, such as output length \citep{flashthink, liu2025learn, aggarwal2025l1}, bug rate of generated code \citep{DBLP:journals/ese/TambonDNKDA25, DBLP:conf/icsm/Gao25}, hallucination rate of output content, and correct function call in tool-use \citep{torl, logic_rl}. To adapt to GRPO, the usual solution is to combine the rewards corresponding to multiple objectives to form a final reward for policy optimization.

Formally, given a dataset $\mathcal{D}$, $x$ is the query and $y$ is the response. For the policy model $\pi_\theta$ parameterized by $\theta$, the likelihood by the policy model $\pi_\theta$ is given by $\pi_\theta(y|x) = \prod_{t=1}^{|y|}\pi_\theta(y_i | x, y_{<t})$. In a multi-reward setting, there are $n$ reward functions $r_1, r_2, \cdots, r_n$ for independent objectives. For a given input-output pair $(x_i, y_j)$, the corresponding reward is denoted as $r_k^{(i,j)} = r_k(x_i, y_j) \in [0,1], k=1,2,\cdots,n$. In the usual practice, the reward $r$ ultimately used for strategy optimization is a convex combination of the various reward component:
\begin{align}
    r_\text{sum}^{(i,j)} = w_1r_1^{(i,j)} + w_2r_2^{(i,j)} + \cdots + w_nr_n^{(i,j)} = \sum_{k} w_k r_k^{(i,j)}, \sum_k w_k = 1, w_k \in [0, 1],
\end{align}
where $w_k$ is the weight hyperparameter corresponding to $r_k$.

For GRPO, each input $x_i$ will sample $G$ rollouts $y_1, y_2, \cdots, y_G$ to calculate the relative advantage:
\begin{align}
    A_\text{sum}^{(i,j)} = \frac{r_\text{sum}^{(i,j)} - \text{mean} \left(\left\{ r_\text{sum}^{(i,j)} \right\}_{j=1}^G\right)}{\text{std} \left(\left\{ r_\text{sum}^{(i,j)} \right\}_{j=1}^G\right)}.\label{reward_combination}
\end{align}

The corresponding policy optimization objective for GRPO can be expressed as:
\begin{align}
    \mathcal{J}_\text{GRPO}(\theta) &= \mathbb{E}_{x_i \sim \mathcal{D},\left\{y_j\right\}_{j=1}^G \sim \pi_\theta(\cdot | x_i)} \notag \\ &\left[ \frac{1}{G} \sum_{j=1}^G \frac{1}{|y_j|} \sum_{t=1}^{|y_j|} \min \left( s_{j,t}(\theta) A_\text{sum}^{(i,j)}, \text{clip}(s_{j,t}(\theta), 1 - \epsilon, 1+\epsilon)A_\text{sum}^{(i,j)} \right) \right],\label{GRPO}
\end{align}
where $s_{j,t}(\theta) = \frac{\pi_\theta (y_{j,t} | x_i, y_{j, <t})}{\pi_{\theta_\text{old}} (y_{j,t} | x_i, y_{j, <t})}$ is the importance sampling ratio and $\epsilon$ is the clipping range. For clarity, we omit the KL divergence term. The corresponding gradient is as follows:
\begin{align}
    \nabla_\theta \mathcal{J}_\text{GRPO}(\theta) &= \mathbb{E}_{x_i \sim \mathcal{D},\left\{y_j\right\}_{j=1}^G \sim \pi_\theta(\cdot | x_i)} \notag \\ \frac{1}{G} \sum_{j=1}^G \frac{1}{|y_j|} \sum_{t=1}^{|y_j|} & \min \left( s_{j,t}(\theta) A_\text{sum}^{(i,j)}, \text{clip}(s_{j,t}(\theta), 1 - \epsilon, 1+\epsilon)A_\text{sum}^{(i,j)} \right) \nabla_\theta \log \pi_\theta (y_{j,t} | x_i, y_{j, <t}).\label{GRPO_gradient}
\end{align}

Another common multi-reward policy optimization method focuses on convex combinations of advantages rather than rewards, such as Group reward-Decoupled Normalization Policy Optimization (GDPO) \citep{gdpo}. Specifically, the independent reward for each objective is calculated as an independent advantage in a manner similar to GRPO, and these advantages are then combined to obtain the advantage used for policy optimization: 
\begin{align}
    A_1^{(i,j)} = \frac{r_1^{(i,j)} - \text{mean} \left(\left\{ r_1^{(i,j)} \right\}_{j=1}^G\right)}{\text{std} \left(\left\{ r_1^{(i,j)} \right\}_{j=1}^G\right)}, \cdots, A_n^{(i,j)} = \frac{r_n^{(i,j)} - \text{mean} \left(\left\{ r_n^{(i,j)} \right\}_{j=1}^G\right)}{\text{std} \left(\left\{ r_n^{(i,j)} \right\}_{j=1}^G\right)}.\label{advantage_combination}
\end{align}

Then, these individual advantages are combined using a similar convex combination method to obtain a single advantage result:
\begin{align}
    A^{(i,j)} = w_1A_1^{(i,j)} + w_2A_2^{(i,j)} + \cdots + w_nA_n^{(i,j)} = \sum_{k} w_k A_k^{(i,j)}, \sum_k w_k = 1, w_k \in [0, 1],\label{advantage_convex}
\end{align}
where $w_k$ is the weight hyperparameter corresponding to $A_k$. Based on $A^{(i,j)}$ and Equation \ref{GRPO}, policy optimization is performed to improve the performance of LLM for multiple objectives. GDPO further utilizes batch-wise advantage normalization to maintain training stability.

\section{Method}
In this section, we will first discuss the shortcomings of the reward combination and advantage combination methods discussed above, and then introduce our proposed DVAO method in detail.

\subsection{Reward Combination and Advantage Combination}
Having introduced both the reward combination method and the advantage combination method, a natural question arises: \textbf{which method produces a more effective gradient signal for policy optimization?} To answer this, we analyze the magnitude of the mean squared advantage, as the policy gradient is directly proportional to the advantage value in Equation \ref{GRPO_gradient}. Specifically, a larger advantage magnitude leads to a larger policy gradient update, which may cause training instability and hinder convergence in the multi-reward setting. To answer this, we have the following proposition.

\begin{proposition}
    For a fixed query $x_i$, let $\hat{\rho}_{kl}^i$ denote the sample correlation between $A_k$ and $A_l$ within the group rollout. The reward combination method and the advantage combination method satisfy:
    \begin{align}
        \frac{1}{G} \sum_{j=1}^G \left(A_\text{sum}^{(i,j)}\right)^2 \ge \frac{1}{G} \sum_{j=1}^G \left(A^{(i,j)}\right)^2 = \frac{1}{G} \sum_{j=1}^G \left(\sum_k w_k A_k^{(i,j)}\right)^2
    \end{align}
    with equality if and only if $\hat{\rho}_{kl}=1$ for all $k \neq l$.
\end{proposition}

This result reveals that the reward combination method, despite its simplicity, produces advantages with larger squared magnitude on average, leading to larger policy gradients. Although the advantage combination method achieves better results in the magnitude of the advantage, it fails to explicitly consider the correlation between multiple rewards. It is essentially equivalent to making a convex combination of the RL optimization objective composed of multiple independent rewards. Full proof is in Appendix~\ref{appendix:proof_of_proposition_1}.

Formally, based on Equation \ref{GRPO_gradient} and Equation \ref{advantage_convex}, without considering clipping range for brevity, we have:
\begin{align}
    \nabla_\theta \mathcal{J}_\text{GRPO}(\theta) &= \mathbb{E}_{x_i \sim \mathcal{D},\left\{y_j\right\}_{j=1}^G \sim \pi_\theta(\cdot | x_i)} \frac{1}{G} \sum_{j=1}^G \frac{1}{|y_j|} \sum_{t=1}^{|y_j|} s_{j,t}(\theta)A^{(i,j)} \nabla_\theta\log \pi_\theta (y_{j,t} | x_i, y_{j, <t}) \notag\\
    &= \sum_k w_k \nabla_\theta \mathcal{J}_\text{GRPO}(\theta)_k,
\end{align}
where $\nabla_\theta\mathcal{J}_\text{GRPO}(\theta)_k$ is the gradient of the RL optimization objective corresponding to $A_k^{(i,j)}$. Therefore, from the perspective of RL gradient, the advantage combination method does not explicitly take into account the correlation between multiple rewards. Furthermore, it is difficult to adjust the training intensity of different RL objectives during dynamic training with fixed hyperparameters of convex combination coefficients $\{w_k\}_{k=1}^n$.

\subsection{Dynamic Variance-adaptive Advantage Optimization}
The above discussion reveals that the reward combination method, despite its simplicity, produces advantages with larger squared magnitude on average, leading to larger policy gradients. While the advantage combination method alleviates this problem by decoupling the normalization of each objective, it still relies on fixed weights and does not explicitly introduce the correlation between multiple rewards, making it difficult to optimize multiple objectives as a whole. This motivates our proposed \textbf{D}ynamic \textbf{V}ariance-adaptive \textbf{A}dvantage \textbf{O}ptimization, namely \textbf{DVAO}, which further adapts the combination weights according to the reward variance of each objective. At the same time, DVAO has a better advantage magnitude than the reward combination method.

Formally, DVAO replaces the fixed combination weights $w_k$ with dynamic variance-adaptive weights $\tilde{w}_k = \frac{w_k\sigma_k^i}{\sum_l w_l \sigma_l^i}$, which up-weights objectives with higher reward variance and down-weights objectives with lower reward variance in a fully dynamic and data-driven manner, where $\sigma_k^i = \text{std} \left(\left\{ r_k^{(i,j)} \right\}_{j=1}^G\right)$ and $\sigma_\text{sum}^i = \text{std} \left(\left\{ r_\text{sum}^{(i,j)} \right\}_{j=1}^G\right)$ are the corresponding group standard deviations. The DVAO advantage is then computed as:
\begin{align}
    A_\text{DVAO}^{(i,j)} = \sum_k \tilde{w}_k A_k^{(i,j)} = \frac{\sum_k w_k \sigma_k^i A_k^{(i,j)}}{\sum_l w_l \sigma_l^i}. \label{DVAO}
\end{align}

To illustrate the advantage of DVAO over the reward combination method in terms of advantage magnitude, we have the following proposition:
\begin{proposition}
\label{proposition_2}
    For a fixed query $x_i$ and rollout group $\{y_j\}_{j=1}^G \sim \pi_\theta(\cdot | x_i)$, the reward combination method produces a pointwise larger advantage magnitude than DVAO:
    \begin{align}
        \left| A_\text{DVAO}^{(i,j)} \right| \le \left| A_\text{sum}^{(i,j)} \right|, \forall j \in \{1, 2, \cdots, G\}
    \end{align}
    with equality if and only if $\Cov\left(r_k^{(i,j)}, r_l^{(i,j)}\right) = \sigma_k^i \sigma_l^i$ for all $k \neq l$, i.e., all reward pairs are perfectly positively correlated within the rollout group.
\end{proposition}

Beyond the pointwise advantage magnitude comparison, we further analyze how DVAO and the advantage combination method differ in their sensitivity to the raw rewards of individual objectives. Full proof is in Appendix~\ref{appendix:proof_of_proposition_2}. This analysis provides a deeper understanding of how DVAO explicitly captures cross-objective interactions, a property that the standard advantage combination method fundamentally lacks. Specifically, we examine the partial derivative of the combined advantage with respect to the raw reward $r_k^{(i,j)}$. This derivative measures how the final advantage responds to a perturbation in the $k$-th objective's reward, reflecting the degree to which each objective influences the overall gradient signal. We have the following proposition:
\begin{proposition}
\label{proposition_3}
    For a fixed query $x_i$, and rollout group $\{y_j\}_{j=1}^G \sim \pi_\theta(\cdot | x_i)$, the sensitivity of the combined advantage with  respect to the $k$-th raw reward $r_k^{(i,j)}$ for the advantage combination method and DVAO are respectively given by:
    \begin{align}
        \frac{\partial A^{(i,j)}}{\partial r_k^{(i,j)}} &= \frac{w_k}{\sigma_k^i} \left( 1 - \frac{1}{G} - \frac{1}{G} \left(A_k^{(i,j)}\right)^2 \right), \\
        \frac{\partial A_\text{DVAO}^{(i,j)}}{\partial r_k^{(i,j)}} &= \frac{\tilde{w}_k}{\sigma_k^i} \left( 1 - \frac{1}{G} - \frac{1}{G} A_\text{DVAO}^{(i,j)} A_k^{(i,j)} \right).
    \end{align}
    While the sensitivity of $A^{(i,j)}$ strictly depends on the isolated advantage of the $k$-th objective, the sensitivity of $A_\text{DVAO}^{(i,j)}$ adaptively depends on the cross-term $A_\text{DVAO}^{(i,j)} A_k^{(i,j)}$, allowing it to aggregate global performance information across all objectives within the rollout group.
\end{proposition}

This result highlights a fundamental difference in the optimization dynamics. In the advantage combination method, the gradient contribution from the $k$-th objective is scaled purely by its own isolated performance $\left(A_k^{(i,j)}\right)^2$, treating the auxiliary objectives as entirely separate tasks. In contrast, DVAO scales the gradient contribution using the cross-interaction term $A_\text{DVAO}^{(i,j)} A_k^{(i,j)}$. This mathematical property proves that DVAO dynamically adjusts the learning signal of the $k$-th objective based on the model's overall multi-objective performance $A_\text{DVAO}^{(i,j)}$ on that specific rollout. Consequently, DVAO automatically modulates the reward sensitivity to reinforce the synergistic alignment of multiple objectives, effectively functioning as a cross-objective, variance-aware regularization mechanism. Full proof is in Appendix~\ref{appendix:proof_of_proposition_3}.

In summary, our proposed DVAO method addresses the fundamental limitations of both standard reward combination and advantage combination methods in multi-reward GRPO. By dynamically adapting combination weights based on the empirical variance of each reward within a rollout group, DVAO achieves two critical theoretical properties. First, as demonstrated in Proposition \ref{proposition_2}, DVAO mitigates the training instability inherent in the raw reward combination method by yielding advantages with a strictly bounded magnitude, preventing overly aggressive policy updates. Second, and perhaps more importantly, Proposition \ref{proposition_3} proves that DVAO goes beyond the naive decoupling of the advantage combination method. By mathematically linking the gradient sensitivity of a single objective to the overall combined advantage $A_\text{DVAO}^{(i,j)} A_k^{(i,j)}$, DVAO introduces an implicit cross-objective regularization mechanism. The learning signal for any individual objective is dynamically modulated by the model's global multi-objective performance on that specific rollout. This context-aware scaling ensures that the policy does not greedily over-optimize a single easy objective at the expense of others, inherently promoting synergistic alignment and a more stable trajectory toward a multi-objective Pareto optimal policy.
 
\section{Experiments}

\subsection{Experimental Setup}
\textbf{Benchmarks.} In this work, we focus specifically on mathematical reasoning and tool-use tasks to evaluate our proposed DVAO algorithm. For mathematical reasoning task, we evaluate models on AIME-2024\footnote{\url{https://huggingface.co/datasets/Maxwell-Jia/AIME_2024}}, AIME-2025\footnote{\url{https://huggingface.co/datasets/yentinglin/aime_2025}}, MATH500 \citep{math500}, OlympiadBench \citep{olympiadbench}, and AMC23\footnote{\url{https://huggingface.co/datasets/AI-MO/aimo-validation-amc}}. In mathematical reasoning tasks, we focus on two main objectives: accuracy and length constrain. For tool-use task, we follow the setup of ToolRL \citep{toolrl} and GDPO, which evaluate models on Berkeley Function Call Leaderboard (BFCL-v4) \citep{bfcl}, a comperhensive benchmark covering a broad range of challanges, including single-step reasoning, multi-step tool-use, real-time execution, irrelevant tool rejection, simultaneous multi-tool selection, and multi-tool execution. In tool-use task, we focus on two main objectives: tool-use correctness and format compliance.

\textbf{Baselines and Models.} We mainly use GRPO \citep{grpo} as the single-reward $r_\text{acc}$ baseline for the comparison. Based on GRPO, we implement the Reward Combination (RC) method and Advantage Combination (AC) method for the multi-reward tasks. For comparison, we include the GDPO \citep{gdpo} algorithm. We use \textit{Qwen3-4B-Base} and \textit{Qwen3-8B-Base} \citep{qwen3} for the mathematical reasoning tasks, and \textit{Qwen2.5-3B-Instruct} and \textit{Qwen2.5-7B-Instruct} \citep{qwen2.5} for the tool-use tasks. For complete implementation details, see Appendix~\ref{appendix:implementation_details}.

\subsection{Main Results}
\label{subsection:main}

\begin{table}[t!]
    \centering
    \caption{Performance comparison across different methods on AIME-2024, AIME-2025, MATH500, OlympaidBench, and AMC23. \textbf{Acc.}: Output Accuracy (\%). \textbf{Len.}: The Rate (\%) of output length not exceeding $l$. DVAO achieves state-of-the-art performance with both scores in average.}
    \scalebox{0.845}{
    \begin{tabular}{l rr rr rr rr rr rr}
    \toprule
    \multirow{2}{*}{\textbf{Method}} & \multicolumn{2}{c}{\textbf{AIME-2024}} & \multicolumn{2}{c}{\textbf{AIME-2025}} & \multicolumn{2}{c}{\textbf{MATH500}} & \multicolumn{2}{c}{\textbf{Olympaid}} & \multicolumn{2}{c}{\textbf{AMC23}} & \multicolumn{2}{c}{\textbf{Average}} \\
    \cmidrule(lr){2-3} \cmidrule(lr){4-5} \cmidrule(lr){6-7} \cmidrule(lr){8-9} \cmidrule(lr){10-11} \cmidrule(lr){12-13}
    & \textbf{Acc.} & \textbf{Len.} & \textbf{Acc.} & \textbf{Len.} & \textbf{Acc.} & \textbf{Len.} & \textbf{Acc.} & \textbf{Len.} & \textbf{Acc.} & \textbf{Len.} & \textbf{Acc.} & \textbf{Len.} \\
    \midrule
    \rowcolor{lightgray}
    \multicolumn{13}{c}{\textbf{Qwen3-4B-Base}} \\
    \addlinespace[0.25em]
    Model & 7.92 & 88.12 & 4.79 & 88.54 & 55.05 & 91.42	& 26.98 & 92.02 & 34.18 & 91.94 & 25.78 & 90.41 \\
    ~+ GRPO & 17.91 & 62.08 & 10.20 & 68.33 & 78.92 &  93.04 & 42.88 & 82.82 & 49.62 & 82.91 & 39.91 & 77.84 \\
    \addlinespace[0.1em]
    \hline
    \addlinespace[0.25em]
    ~+ RC & 14.58 & 92.50 & 9.38 & 95.42 & 78.31 & 98.91 & 41.18 & 97.43 & 51.50 & 97.67 & 38.99 & 96.39 \\
    ~+ AC & 16.25 & 91.04 & 9.38 & 95.21 & 77.65 & 98.69 & 41.02 & 98.08 & 49.47 & 98.12 & 38.75 & 96.23 \\
    ~+ GDPO & 2.08 & 95.83 & 3.75 & 96.46 & 30.06 & 99.52 & 16.56 & 98.29 & 14.60 & 98.95 & 13.41 & 97.81 \\
    \addlinespace[0.1em]
    \rowcolor{lightblue}
    ~\textbf{+ DVAO} & \textbf{16.87} & \textbf{100.0} & \textbf{13.54} & \textbf{99.79} & \textbf{81.36} & \textbf{99.94} & \textbf{45.63} & \textbf{99.96} & \textbf{53.53} & \textbf{99.85} & \textbf{42.19} & \textbf{99.91} \\
    \midrule
    \rowcolor{lightgray}
    \multicolumn{13}{c}{\textbf{Qwen3-8B-Base}} \\
    \addlinespace[0.25em]
    Model & 11.87 & 89.79 & 8.96 & 93.33 & 69.20 & 97.44 & 34.95 & 94.66 & 41.56 & 96.16 & 33.31 & 94.28 \\
    ~+ GRPO & 29.58 & 40.00 & 24.58 & 52.08 & 87.93 & 89.61 & 55.57 & 70.81 & 65.21 & 64.83 & 52.57 & 63.47 \\
    \addlinespace[0.1em]
    \hline
    \addlinespace[0.25em]
    ~+ RC & 21.04 & 97.08 & 16.25 & 99.38 & 84.97 & 99.59 & 49.73 & 98.48 & 59.33 & 99.02 & 46.26 & 98.71 \\
    ~+ AC & 20.41 & 97.71 & 15.62 & 98.96 & 84.42 & 99.58 & 48.52 & 98.93 & 58.13 & 99.02 & 45.42 & 98.84 \\
    ~+ GDPO & 1.67 & 100.0 & 0.00 & 100.0 & 35.07 & \textbf{100.0} & 9.15 & \textbf{99.96} & 27.56 & \textbf{100.0} & 14.69 & \textbf{99.99} \\
    \addlinespace[0.1em]
    \rowcolor{lightblue}
    ~\textbf{+ DVAO} & \textbf{21.87} & \textbf{100.0} & \textbf{18.33} & \textbf{100.0} & \textbf{86.10} & 99.99 & \textbf{50.62} & 99.76 & \textbf{60.54} & 99.85 & \textbf{47.49} & 99.92 \\
    \bottomrule
    \end{tabular}
    }
    \label{tab:main_math}
\end{table}
\begin{table}[t!]
    \centering
    \caption{Performance comparison across different methods on Live, Non-Live, and Multi-Turn in BFCL-v4. \textbf{Acc.}: Output Accuracy (\%). \textbf{Format.}: The Rate (\%) of output content conforming to the required format. DVAO achieves state-of-the-art performance with both scores in average.}
    \scalebox{0.845}{
    \begin{tabular}{l rr rr rr rr}
    \toprule
    \multirow{2}{*}{\textbf{Method}} & \multicolumn{2}{c}{\textbf{Live}} & \multicolumn{2}{c}{\textbf{Non-Live}} & \multicolumn{2}{c}{\textbf{Multi-Turn}} & \multicolumn{2}{c}{\textbf{Average}} \\
    \cmidrule(lr){2-3} \cmidrule(lr){4-5} \cmidrule(lr){6-7} \cmidrule(lr){8-9}
    & \textbf{Acc.} & \textbf{Format.} & \textbf{Acc.} & \textbf{Format.} & \textbf{Acc.} & \textbf{Format.} & \textbf{Acc.} & \textbf{Format.} \\
    \midrule
    \rowcolor{lightgray}
    \multicolumn{9}{c}{\textbf{Qwen2.5-3B-Instruct}} \\
    \addlinespace[0.25em]
    Model & 59.44 & 11.02 & 47.10 & 4.89 & 1.75 & 1.08 & 36.10 & 5.66 \\
    ~+ GRPO & 59.73 & 10.35 & 47.94 & 5.18 & 2.00 & 1.29 & 36.56 & 5.61 \\
    \addlinespace[0.1em]
    \hline
    \addlinespace[0.25em]
    ~+ RC & 67.51 & 61.71 & 79.94 & 83.45 & 5.62 & 36.29 & 51.02 & 60.48 \\
    ~+ AC & 69.43 & 67.04 & 82.35 & 84.82 & 8.62 & 42.20 & 53.47 & 64.69 \\
    ~+ GDPO & 67.73 & 65.08 & 82.08 & 84.53 & 8.38 & 48.04 & 52.73 & 65.88 \\
    \addlinespace[0.1em]
    \rowcolor{lightblue}
    ~\textbf{+ DVAO} & \textbf{72.73} & \textbf{77.44} & \textbf{84.75} & \textbf{95.11} & \textbf{12.50} & \textbf{57.40} & \textbf{56.66} & \textbf{76.65} \\
    \midrule
    \rowcolor{lightgray}
    \multicolumn{9}{c}{\textbf{Qwen2.5-7B-Instruct}} \\
    \addlinespace[0.25em]
    Model & 62.92 & 0.0 & 69.56 & 0.0 & 11.00 & 0.0 & 47.83 & 0.00 \\
    ~+ GRPO & 68.76 & 0.0 & 81.65 & 0.0 & 6.38 & 0.0 & 52.26 & 0.00 \\
    \addlinespace[0.1em]
    \hline
    \addlinespace[0.25em]
    ~+ RC & 75.06 & \textbf{87.58} & 85.33 & \textbf{96.11} & 14.75 & 45.56 & 58.38 & 76.42 \\
    ~+ AC & 63.80 & 67.39 & 56.21 & 85.40 & 12.75 & 51.33 & 44.25 & 68.04 \\
    ~+ GDPO & 76.17 & 68.90 & 86.73 & 85.83 & 17.50 & 49.63 & 60.13 & 68.12 \\
    \addlinespace[0.1em]
    \rowcolor{lightblue}
    ~\textbf{+ DVAO} & \textbf{79.68} & 77.93 & \textbf{87.06} & 95.11 & \textbf{22.25} & \textbf{64.58} & \textbf{63.00} & \textbf{79.21} \\
    \bottomrule
    \end{tabular}
    }
    \label{tab:main_tool}
\end{table}

Tables~\ref{tab:main_math} and Table~\ref{tab:main_tool} summarize the performance across all methods and model scales. DVAO achieves the highest average accuracy and near-perfect length/format compliance simultaneously across both tasks and model scales, while every baseline method sacrifices one dimension for the other. On math reasoning, RC and AC trade accuracy for length compliance, and GDPO achieves near-perfect length compliance at the cost of the lowest accuracy among all methods. On tool-use, DVAO leads both accuracy and format compliance by a substantial margin. Notably, DVAO remains the only method to achieve the highest score on both dimensions simultaneously at both scales, whereas other methods improve one dimension at the expense of the other---GRPO shows near-zero format compliance on both tool-use models, and AC on 7B actually underperforms the base model in accuracy. Importantly, all methods share the same equal-weight initialization, so the consistent advantage of DVAO stems from its adaptive mechanism rather than superior initial hyperparameter choice, a conclusion reinforced by the Pareto frontier analysis in Section~\ref{subsection:pareto_frontiers} where DVAO dominates across the entire weight sweep.

\subsection{Training Dynamics}
\label{subsection:training_dynamics}

\begin{figure*}[t]
    \centering
    \begin{minipage}{0.31\textwidth}
        \centering
        \resizebox{\linewidth}{!}{\includegraphics{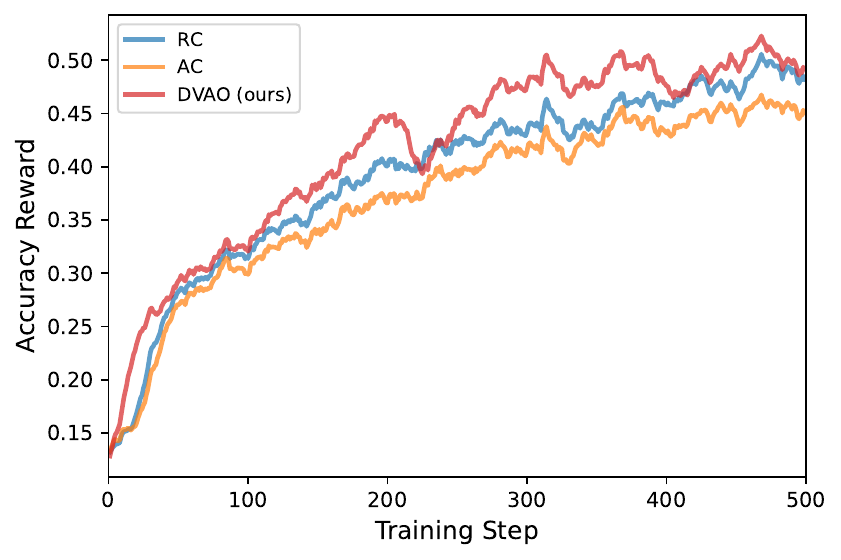}}
        \vspace{4pt}
        \resizebox{\linewidth}{!}{\includegraphics{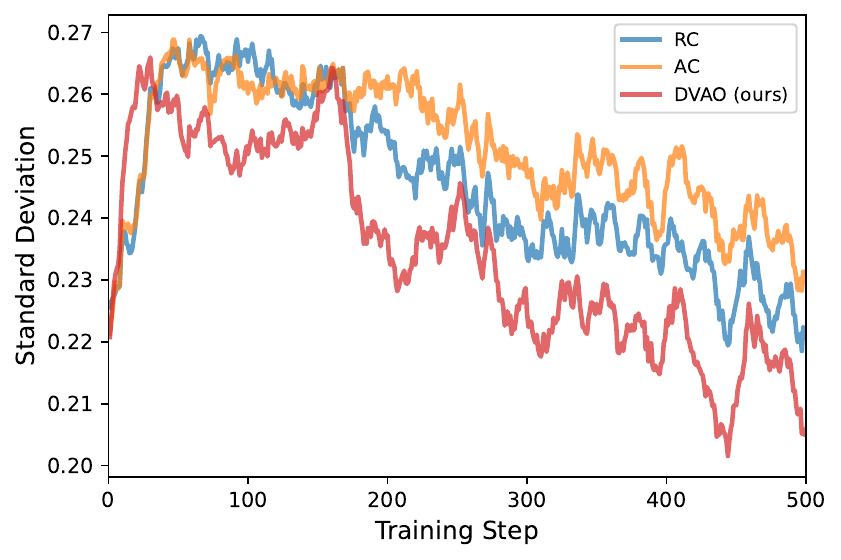}}
    \end{minipage}
    \hfill
    \begin{minipage}{0.31\textwidth}
        \centering
        \resizebox{\linewidth}{!}{\includegraphics{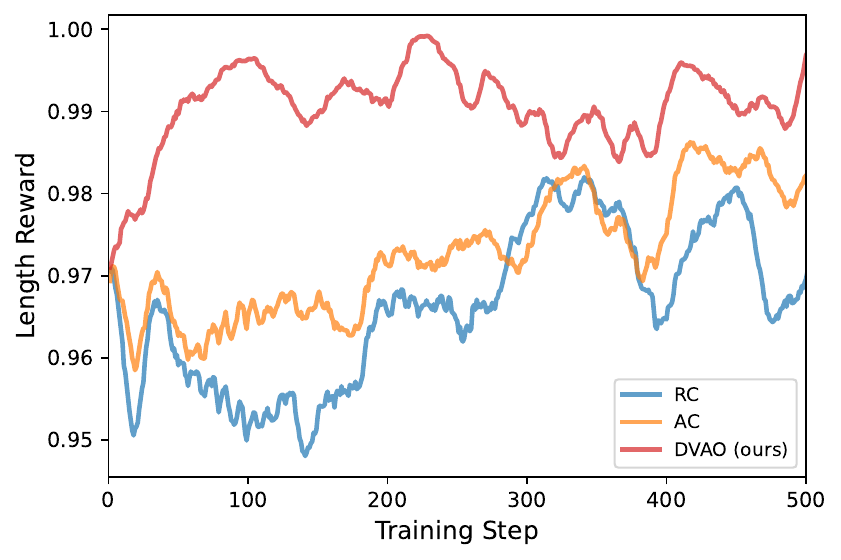}}
        \vspace{4pt}
        \resizebox{\linewidth}{!}{\includegraphics{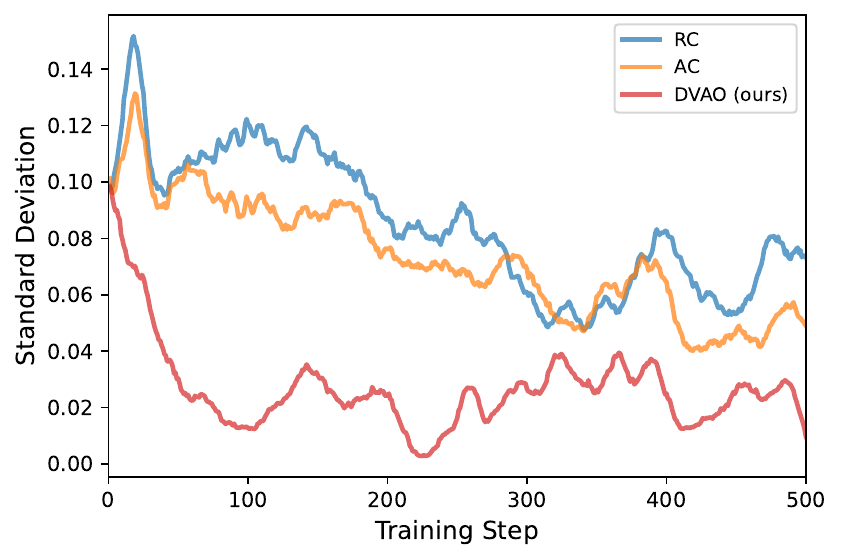}}
    \end{minipage}
    \hfill
    \begin{minipage}{0.31\textwidth}
        \centering
        \resizebox{\linewidth}{!}{\includegraphics{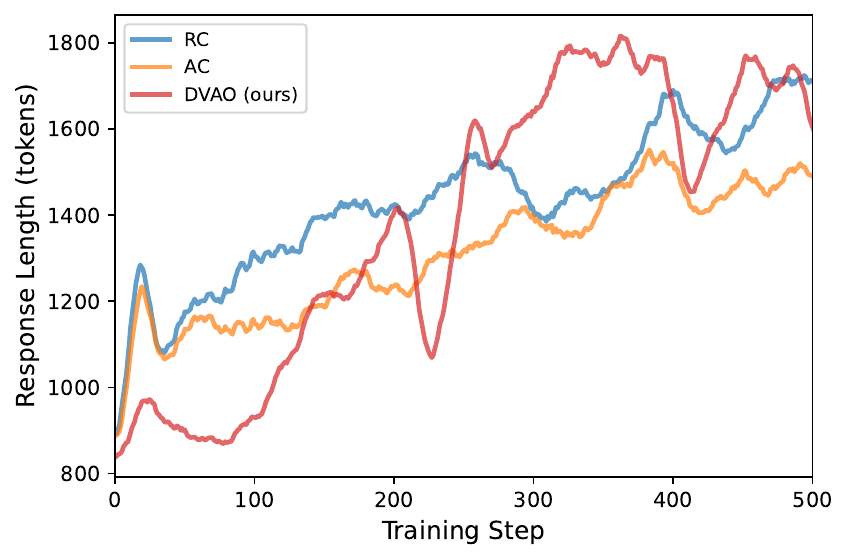}}
    \end{minipage}
    \caption{Training dynamics on \textit{Qwen3-4B-Base}. Left: accuracy reward (top=mean, bottom=std). Middle: length reward (top=mean, bottom=std). Right: average response length.}
    \label{fig:training-dynamics-4b}
\end{figure*}

\begin{figure*}[t]
    \centering
    \begin{minipage}{0.31\textwidth}
        \centering
        \resizebox{\linewidth}{!}{\includegraphics{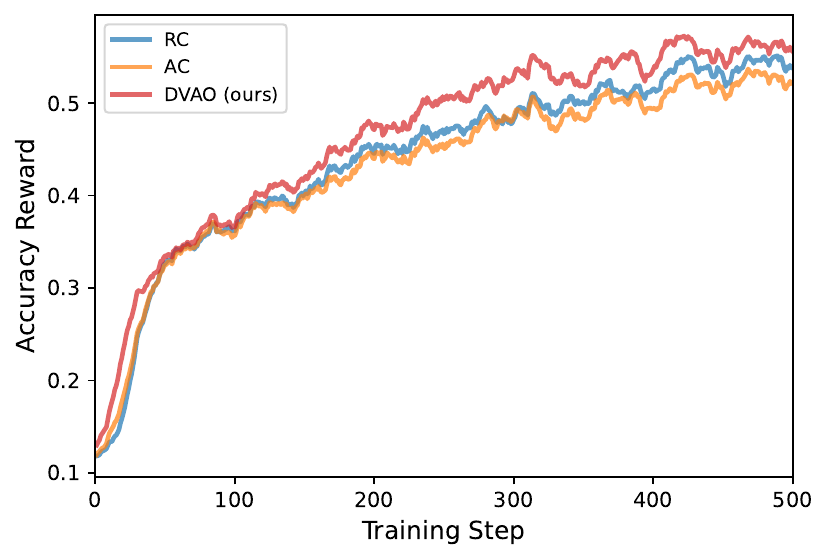}}
        \vspace{4pt}
        \resizebox{\linewidth}{!}{\includegraphics{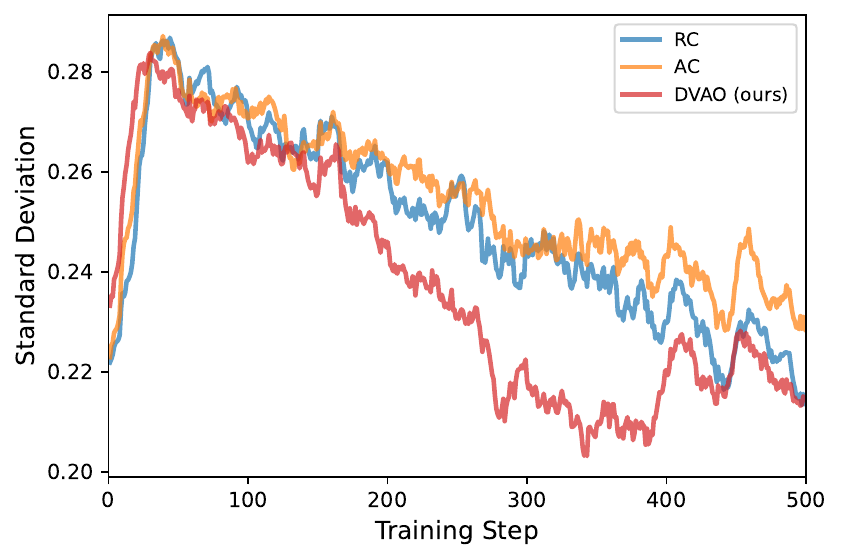}}
    \end{minipage}
    \hfill
    \begin{minipage}{0.31\textwidth}
        \centering
        \resizebox{\linewidth}{!}{\includegraphics{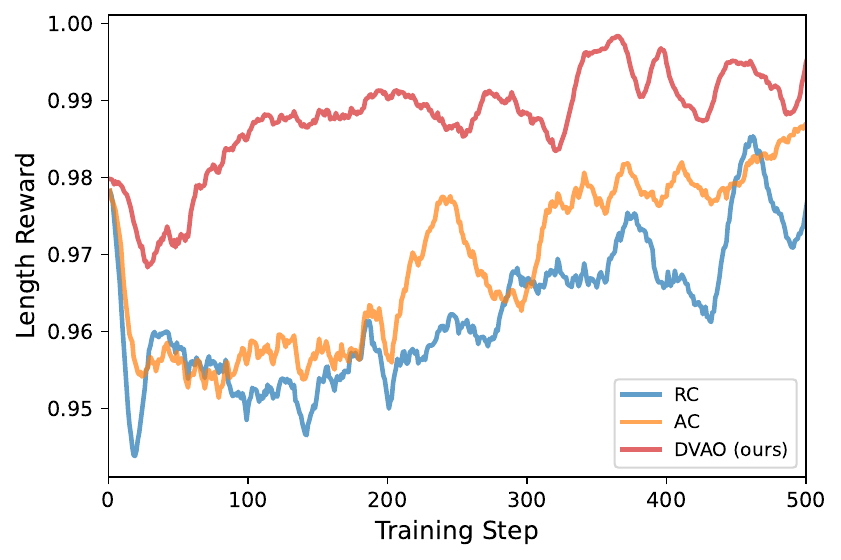}}
        \vspace{4pt}
        \resizebox{\linewidth}{!}{\includegraphics{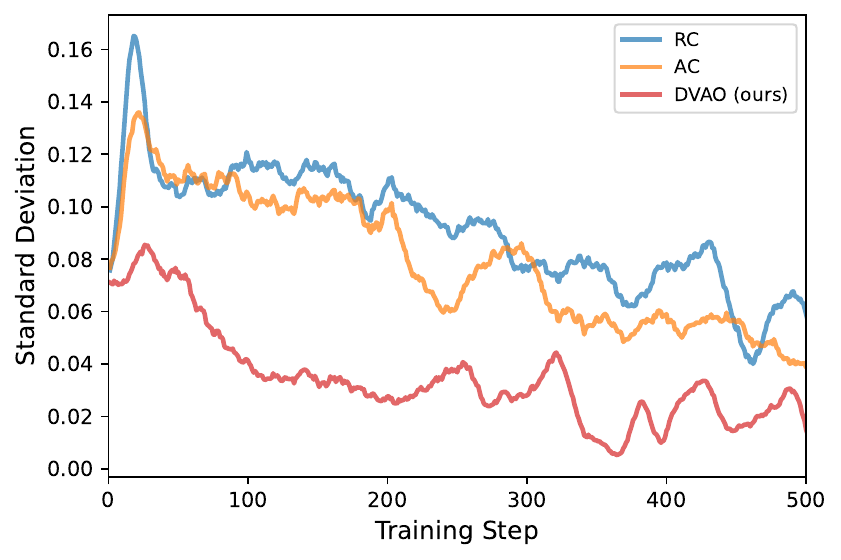}}
    \end{minipage}
    \hfill
    \begin{minipage}{0.31\textwidth}
        \centering
        \resizebox{\linewidth}{!}{\includegraphics{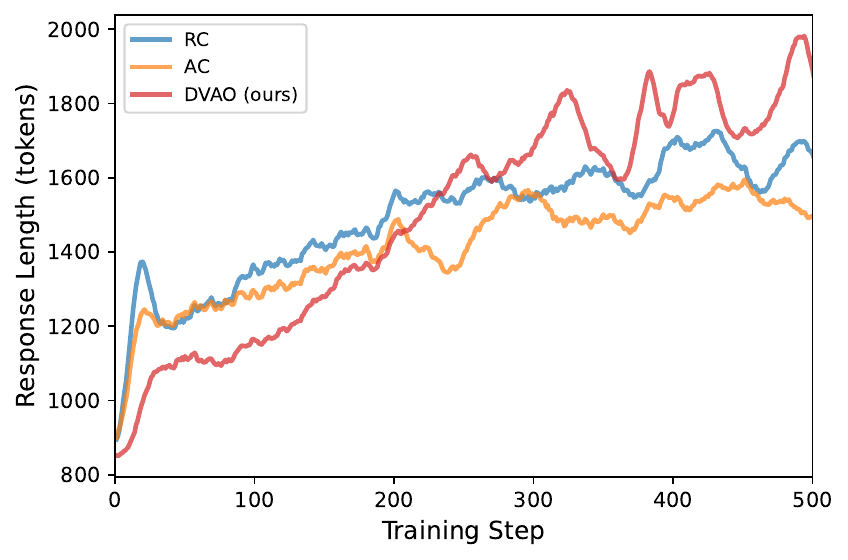}}
    \end{minipage}
    \caption{Training dynamics on \textit{Qwen3-8B-Base}. Left: accuracy reward (top=mean, bottom=std). Middle: length reward (top=mean, bottom=std). Right: average response length.}
    \label{fig:training-dynamics-8b}
\end{figure*}

To understand how DVAO shapes the optimization trajectory, we visualize the evolution of accuracy reward, length reward, and response length throughout training on both \textit{Qwen3-4B-Base} and \textit{Qwen3-8B-Base} (Figure~\ref{fig:training-dynamics-4b} and~\ref{fig:training-dynamics-8b}). All curves are smoothed with a centered moving average (window~$=$~15).

\noindent\textbf{Accuracy reward.} Across both model scales, DVAO consistently achieves the highest accuracy reward while suppressing its variance most effectively. All methods start from a similar low baseline and rise steadily throughout training. DVAO's accuracy reward curve stays above all baselines at every stage, with the margin widening on the larger model. More importantly, the standard deviation of accuracy rewards under DVAO declines more sharply than all baselines. On both 4B and 8B, DVAO's accuracy standard deviation drops to the lowest final value among all methods, while AC consistently exhibits the highest variance throughout training. This combination of higher mean accuracy and lower variance indicates that adaptive variance normalization yields both stronger task performance and more stable gradients, consistent with Proposition~\ref{proposition_2} which guarantees that DVAO's advantage magnitude remains bounded and well-scaled throughout training.

\noindent\textbf{Length reward.} DVAO drives the length reward closest to the target value of 1.0 and exhibits the most dramatic variance collapse. On both model scales, DVAO's length reward rises quickly and stabilizes near the target, while RC fluctuates more noticeably and settles at a visibly lower level. The length reward standard deviation under DVAO shows a far steeper decline than any baseline. For 4B, DVAO's standard deviation drops to a fraction of the RC and AC final values, which remain clustered together at significantly higher levels. For 8B, the gap is even more pronounced, with DVAO's standard deviation approaching near-zero while baselines retain substantially more variance. This variance-balancing mechanism prevents either reward channel from dominating the gradient, enabling more stable convergence to the target length reward. The pronounced std collapse under DVAO directly reflects the cross-objective regularization effect described in Proposition~\ref{proposition_3}, where the adaptive normalization couples the accuracy and length objectives to prevent either from overwhelming the combined advantage signal.

\noindent\textbf{Response length.} All methods start from a similar initial response length of around 800 tokens. DVAO drives the fastest and most sustained growth, reaching the highest final length on both model scales. RC achieves comparable final lengths, while AC exhibits the slowest growth. Notably, DVAO's response length curves on 4B and 8B display more visible oscillation compared to the smoother trajectories of RC and AC, which reflects its more aggressive length reward optimization and explores longer reasoning traces more dynamically. Despite this oscillation, the mean response length still converges to a higher plateau, confirming that DVAO's bounded advantage signal prevents runaway exploration while still encouraging productive length growth.

\subsection{Pareto Frontiers}
\label{subsection:pareto_frontiers}

\begin{figure}
    \centering
    \begin{subfigure}{0.48\textwidth}
        \centering
        \includegraphics[width=\linewidth]{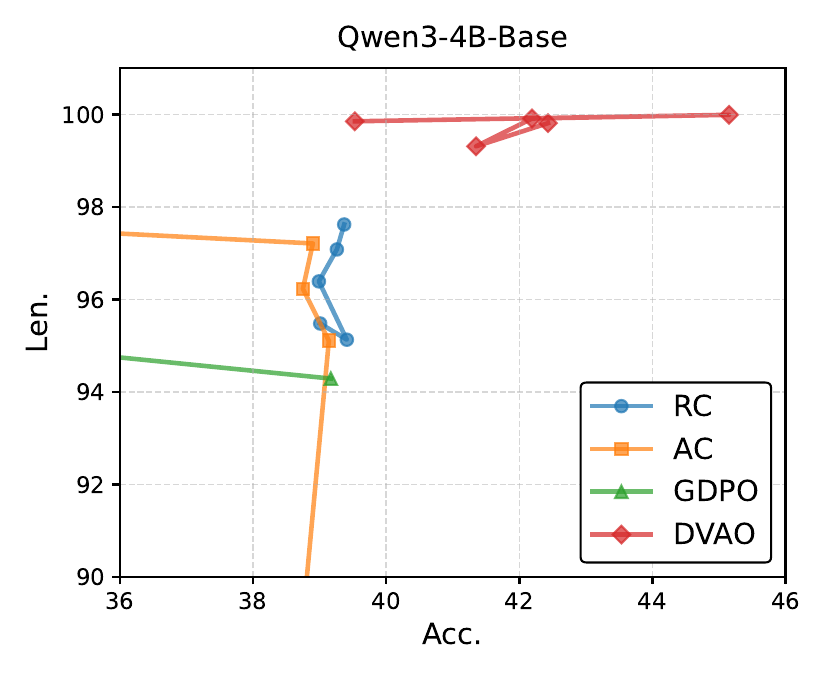}
        \caption{Mathematical Reasoning Task (\textit{Qwen3-4B-Base})}
        \label{fig:pareto-math}
    \end{subfigure}
    \hfill
    \begin{subfigure}{0.48\textwidth}
        \centering
        \includegraphics[width=\linewidth]{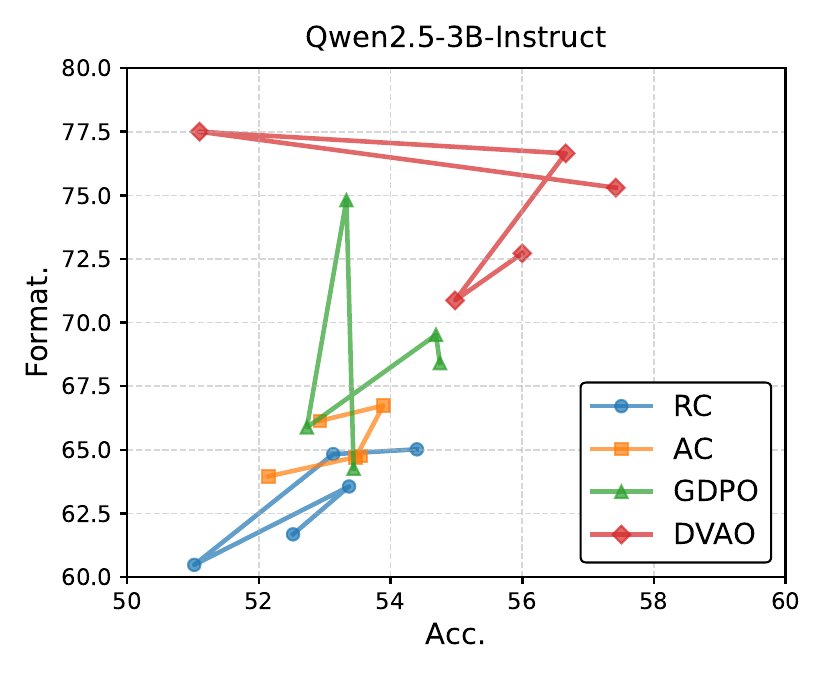}
        \caption{Tool-Use Task (\textit{Qwen2.5-3B-Instruct})}
        \label{fig:pareto-tool}
    \end{subfigure}
    \caption{Pareto frontier of accuracy vs.\ length/format compliance across methods. DVAO consistently dominates the upper-right region.}
    \label{fig:pareto}
\end{figure}

While Section~\ref{subsection:main} reports the trained single checkpoint per method, optimizing with a single weight $w_k$ only reveals one point on the accuracy--length trade-off curve. To fully characterize how each method balances correctness and conciseness, we sweep the accuracy weight $w_1$ (with length weight or format weight $w_2 = 1 - w_1$) across a range of values $\{0.1, 0.3, 0.5, 0.7, 0.9\}$ and plot the resulting Pareto frontier in Figure~\ref{fig:pareto-math} (\textit{Qwen3-4B-Base}, Mathematical Reasoning Task) and Figure~\ref{fig:pareto-tool} (\textit{Qwen2.5-3B-Instruct}, Tool-Use Task).

DVAO consistently achieves superior trade-offs, dominating the frontier across both tasks by maintaining high auxiliary compliance (length/format) across the entire accuracy range. In contrast, fixed-weight baselines exhibit distinct failure modes: RC saturates quickly, AC suffers from severe instability, and GDPO fluctuates incoherently. This confirms that without adaptive normalization, uncontrolled advantage scaling prevents effective trade-off navigation. Furthermore, DVAO's dominance is especially pronounced on the complex math task, shifting the entire frontier significantly above baselines. This suggests that dynamic variance-balancing is crucial for multi-step reasoning, as it prevents easier objectives (length) from overwhelming harder learning signals (accuracy) during training.

\section{Related Work}

\textbf{Advancements in GRPO and Reasoning Models}\quad The shift from PPO \citep{ppo} to GRPO \citep{grpo} has significantly streamlined the post-training pipeline for Large Language Models by eliminating the need for a heavily parameterized value model. This efficiency has been instrumental in the development of state-of-the-art reasoning models like DeepSeek-R1 \citep{deepseek-r1}. Recently, several variants have been proposed to further enhance GRPO’s stability and efficiency. GSPO \citep{gspo} shifts the importance ratio calculation from the token level to the sequence level to mitigate variance. DAPO \citep{dapo} introduces dynamic sampling and token-level policy gradients to accelerate convergence. To address the issue of length explosion during reasoning, methods like GFPO \citep{gfpo} and DLER \citep{dler} incorporate length-based heuristics, such as filtering responses by reward-per-token ratios or applying simple truncation penalties. While these extensions improve specific aspects of the GRPO framework, they primarily focus on single-reward maximization or rely on rigid heuristics, which struggle to generalize to complex, multi-dimensional alignment tasks.

\textbf{Multi-Reward Reinforcement Learning in LLMs}\quad Integrating multiple, often conflicting, reward signals is increasingly vital for practical LLM deployments, ranging from balancing diverse human preferences \citep{alarm, jang2023personalized} to enforcing length efficiency \citep{gfpo, dler, luo2025o1} and strict formatting constraints in agentic tool-use \citep{toolrl, xlam}. To simultaneously optimize these diverse objectives, standard practices typically rely on Reward Combination (RC) or Advantage Combination (AC). While RC directly scalarizes raw rewards but frequently suffers from magnitude explosion, AC-based methods like GDPO \citep{gdpo} independently normalize each reward into an advantage before applying a static convex combination. Although GDPO mitigates extreme gradients, its reliance on fixed hyperparameters completely isolates the objectives during normalization. Our proposed DVAO diverges fundamentally from these approaches: by introducing a dynamic, variance-adaptive weighting scheme, DVAO strictly bounds advantage magnitudes while explicitly modeling cross-objective correlations, enabling a self-adaptive regularization mechanism that seamlessly scales to multiple objectives without manual tuning.
\section{Conclusion}

In this work, we identify the fundamental theoretical and practical limitations of standard scalarization techniques—namely Reward Combination and Advantage Combination—for multi-reward GRPO. To address the issues of magnitude explosion and objective isolation, we introduce Dynamic Variance-adaptive Advantage Optimization. By dynamically adjusting combination weights based on the empirical variance of each objective within a rollout group, DVAO explicitly up-weights learning signals from high-variance objectives while suppressing low-variance noise. Empirical evaluations across comprehensive mathematical reasoning and tool-use benchmarks confirm that DVAO achieves a superior Pareto optimal policy, seamlessly balancing accuracy with length and format constraints without relying on fixed hyperparameters. Future work will explore scaling the DVAO framework to environments with a larger number of conflicting reward functions and extending the variance-adaptive mechanism to broader alignment paradigms.

\bibliography{neurips_2026}
\bibliographystyle{plainnat}

\appendix
\BeforeBeginEnvironment{proposition}{\addtocounter{proposition}{-3}}

\section{Proof of Proposition 1}
\label{appendix:proof_of_proposition_1}

\begin{proposition}
    For a fixed query $x_i$, let $\hat{\rho}_{kl}^i$ denote the sample correlation between $A_k$ and $A_l$ within the group rollout. The reward combination method and the advantage combination method satisfy:
    \begin{align}
        \frac{1}{G} \sum_{j=1}^G \left(A_\text{sum}^{(i,j)}\right)^2 \ge \frac{1}{G} \sum_{j=1}^G \left(A^{(i,j)}\right)^2 = \frac{1}{G} \sum_{j=1}^G \left(\sum_k w_k A_k^{(i,j)}\right)^2 \notag
    \end{align}
    with equality if and only if $\hat{\rho}_{kl}=1$ for all $k \neq l$.
\end{proposition}
\begin{proof}
    For a fixed query $x_i$, since $A_k^{(i,j)}$ is computed by normalizing $r_k^{(i,j)}$ within the rollout group, we have by definition:
    \begin{align}
        \frac{1}{G}\sum_{j=1}^G A_k^{(i,j)} = 0, \frac{1}{G}\sum_{j=1}^G \left( A_k^{(i,j)} \right)^2 = 1.\notag
    \end{align}
    For the reward combination method, similarly by definition of normalization:
    \begin{align}
        \frac{1}{G} \sum_{j=1}^G \left( A_\text{sum}^{(i,j)} \right)^2 = 1.\notag
    \end{align}
    For the advantage combination:
    \begin{align}
        \frac{1}{G} \sum_{j=1}^G \left(A^{(i,j)}\right)^2 &= \frac{1}{G} \sum_{j=1}^G \left(\sum_k w_k A_k^{(i,j)}\right)^2 \notag\\
        &= \sum_k w_k^2 \cdot \frac{1}{G} \sum_{j=1}^G \left( A_k^{(i,j)} \right)^2 + 2 \sum_{k < l} w_k w_l \cdot \frac{1}{G} \sum_{j=1}^G A_k^{(i,j)}A_l^{(i,j)} \notag\\
        &= \sum_k w_k^2 + 2 \sum_{k<l} w_k w_l \hat{\rho}_{kl}^i \notag\\
        &= \left(\sum_k w_k\right)^2 - 2\sum_{k<l} w_k w_l \left(1 - \hat{\rho}_{kl}^i\right) \notag\\
        &= 1 - 2\sum_{k<l} w_k w_l \left(1 - \hat{\rho}_{kl}^i\right) \le 1, \notag
    \end{align}
    which completes the proof.
\end{proof}

\section{Proof of Proposition 2}
\label{appendix:proof_of_proposition_2}

\BeforeBeginEnvironment{proposition}{\addtocounter{proposition}{3}}

\begin{proposition}
    For a fixed query $x_i$ and rollout group $\{y_j\}_{j=1}^G \sim \pi_\theta(\cdot | x_i)$, the reward combination method produces a pointwise larger advantage magnitude than DVAO:
    \begin{align}
        \left| A_\text{DVAO}^{(i,j)} \right| \le \left| A_\text{sum}^{(i,j)} \right|, \forall j \in \{1, 2, \cdots, G\} \notag
    \end{align}
    with equality if and only if $\Cov\left(r_k^{(i,j)}, r_l^{(i,j)}\right) = \sigma_k^i \sigma_l^i$ for all $k \neq l$, i.e., all reward pairs are perfectly positively correlated within the rollout group.
\end{proposition}
\begin{proof}
    For a fixed query $x_i$, and rollout group $\{y_j\}_{j=1}^G \sim \pi_\theta (\cdot | x_i)$, based on the definitions of the reward combination method, we first establish the following key identity:
    \begin{align}
        \sigma_\text{sum}^i A_\text{sum}^{(i,j)} = \sum_k w_k \sigma_k^i A_k^{(i,j)},\notag
    \end{align}
    which follows from:
    \begin{align}
        \sigma_\text{sum}^i A_\text{sum}^{(i,j)} = r_\text{sum}^{(i,j)} - \frac{1}{G} \sum_{j'}r_\text{sum}^{(i,j')} = \sum_k w_k \left( r_k^{(i,j)} - \frac{1}{G}\sum_{j'}r_k^{(i,j')} \right) = \sum_k w_k \sigma_k^i A_k^{(i,j)}.\notag
    \end{align}
    
    We next show that $\sigma_\text{sum}^i \le \sum_k w_k \sigma_k^i$. By expanding the variance of the combined reward:
    \begin{align}
        \left(\sigma_\text{sum}^i\right)^2 = \Var\left( \sum_k w_k r_k^{(i,j)} \right) = \sum_k w_k^2 \left(\sigma_k^i\right)^2 + 2 \sum_{k < l} w_k w_l \Cov\left(r_k^{(i,j)}, r_l^{(i,j)}\right).\notag
    \end{align}
    By the Cauchy-Schwarz inequality, $\Cov\left(r_k^{(i,j)}, r_l^{(i,j)}\right) \le \sigma_k^i \sigma_l^i$, and therefore:
    \begin{align}
        \left(\sigma_\text{sum}^i\right)^2 \le \sum_k w_k^2 \left(\sigma_k^i\right)^2 + 2\sum_{k < l} w_k w_l \sigma_k^i \sigma_l^i = \left(\sum_k w_k \sigma_k^i\right)^2.\notag
    \end{align}
    Taking the square root on both sides yields $\sigma_\text{sum}^i \le \sum_k w_k \sigma_k^i$. Finally, taking the absolute value of the above identity and dividing both sides by $\sigma_\text{sum}^i$:
    \begin{align}
        \left|A_\text{sum}^{(i,j)}\right| = \frac{\left|\sum_k w_k \sigma_k^i A_k^{(i,j)}\right|}{\sigma_\text{sum}^i} \ge \frac{\left|\sum_k w_k \sigma_k^i A_k^{(i,j)}\right|}{\sum_kw_k \sigma_k^i} = \left|\sum_k \tilde{w}_k A_k^{(i,j)}\right| = \left|A_\text{DVAO}^{(i,j)}\right|,\notag
    \end{align}
    where the inequality follows from $\sigma_\text{sum}^i \le \sum_k w_k \sigma_k^i$, and the last equality follows from the definition of $A_\text{DVAO}^{(i,j)}$.
\end{proof}

\section{Proof of Proposition 3}
\label{appendix:proof_of_proposition_3}

\begin{proposition}
    For a fixed query $x_i$, and rollout group $\{y_j\}_{j=1}^G \sim \pi_\theta(\cdot | x_i)$, the sensitivity of the combined advantage with  respect to the $k$-th raw reward $r_k^{(i,j)}$ for the advantage combination method and DVAO are respectively given by:
    \begin{align}
        \frac{\partial A^{(i,j)}}{\partial r_k^{(i,j)}} &= \frac{w_k}{\sigma_k^i} \left( 1 - \frac{1}{G} - \frac{1}{G} \left(A_k^{(i,j)}\right)^2 \right), \notag \\
        \frac{\partial A_\text{DVAO}^{(i,j)}}{\partial r_k^{(i,j)}} &= \frac{\tilde{w}_k}{\sigma_k^i} \left( 1 - \frac{1}{G} - \frac{1}{G} A_\text{DVAO}^{(i,j)} A_k^{(i,j)} \right). \notag
    \end{align}
    While the sensitivity of $A^{(i,j)}$ strictly depends on the isolated advantage of the $k$-th objective, the sensitivity of $A_\text{DVAO}^{(i,j)}$ adaptively depends on the cross-term $A_\text{DVAO}^{(i,j)} A_k^{(i,j)}$, allowing it to aggregate global performance information across all objectives within the rollout group.
\end{proposition}
\begin{proof}
    For a fixed query $x_i$, the standard group normalization advantage for the $k$-th objective is $A_k^{(i,j)} = \frac{r_k^{(i,j)} - \mu_k^i}{\sigma_k^i}$, where $\mu_k^i = \frac{1}{G} \sum_{j=1}^G r_k^{(i,j)}$. Using the standard properties of the sample mean and standard deviation, the partial derivatives are $\frac{\partial \mu_k^i}{\partial r_k^{(i,j)}} = \frac{1}{G}$ and $\frac{\partial \sigma_k^i}{\partial r_k^{(i,j)}} = \frac{A_k^{(i,j)}}{G}$. Consequently, the derivative of the individual advantage is:
    \begin{align}
        \frac{\partial A_k^{(i,j)}}{\partial r_k^{(i,j)}} = \frac{1}{\sigma_k^i} \left( 1 - \frac{1}{G} - \frac{1}{G} \left(A_k^{(i,j)}\right)^2 \right).\notag
    \end{align}
    For the advantage combination method, since $A^{(i,j)} = \sum_l w_l A_l^{(i,j)}$ and the rewards across different objectives are treated independently in their respective normalizations, applying the chain rule directly yields:
    \begin{align}
        \frac{\partial A^{(i,j)}}{\partial r_k^{(i,j)}} = w_k \frac{\partial A_k^{(i,j)}}{\partial r_k^{(i,j)}} = \frac{w_k}{\sigma_k^i} \left( 1 - \frac{1}{G} - \frac{1}{G} \left(A_k^{(i,j)}\right)^2 \right).\notag
    \end{align}
    For DVAO, we can rewrite the advantage as $A_\text{DVAO}^{(i,j)} = \frac{\sum_l w_l (r_l^{(i,j)} - \mu_l^i)}{S^i}$, where the denominator $S^i = \sum_k w_k \sigma_k^i$. By applying the quotient rule with respect to $r_k^{(i,j)}$, we obtain:
    \begin{align}
        \frac{\partial A_\text{DVAO}^{(i,j)}}{\partial r_k^{(i,j)}} &= \frac{w_k \left( 1 - \frac{1}{G} \right) S^i - \left[ \sum_l w_l (r_l^{(i,j)} - \mu_l^i) \right] w_k \frac{\partial \sigma_k^i}{\partial r_k^{(i,j)}}}{\left(S^i\right)^2} \notag\\
        &= \frac{w_k \left( 1 - \frac{1}{G} \right) S^i - \left( S^i A_\text{DVAO}^{(i,j)} \right) w_k \frac{A_k^{(i,j)}}{G}}{\left(S^i\right)^2} \notag\\
        &= \frac{w_k}{S^i} \left( 1 - \frac{1}{G} - \frac{1}{G} A_\text{DVAO}^{(i,j)} A_k^{(i,j)} \right).\notag
    \end{align}
    Substituting the definition $\frac{w_k}{S^i} = \frac{\tilde{w}_k}{\sigma_k^i}$ into the equation completes the proof.
\end{proof}

\section{Implementation Details}
\label{appendix:implementation_details}

For training dataset, we use DAPO-MATH-17K\footnote{\url{https://huggingface.co/datasets/BytedTsinghua-SIA/DAPO-Math-17k}} for mathematical reasoning task, which consists of 17k prompts, each paired with an interger as the answer, and the same training dataset from ToolRL, which consists of 2k samples from ToolACE \citep{toolace}, 1k samples from Hammer \citep{hammer}, and 1k samples from xLAM \citep{xlam}, each training instance contains a question and its corresponding groud-truth tool calls. For the reward design, we use the accuracy reward $r_\text{acc}$ and the length reward $r_\text{length}$ for mathematical reasoning task, and use the accuracy reward $r_\text{acc}$ and the format reward $r_\text{format}$ for tool-use task. Specifically, $r_\text{length} \in \{0, 1\}$ checks whether the model's output remains within the target length $l$, which is set 4,000 tokens for all remaining experiments, and $r_\text{format} \in \{0,1\}$ checks whether the model't output satisfies the required structure and contains all necessary fields in the correct order. All rewards are constrained within the range $[0,1]$ to maintain consistency with the preceding discussion, while the final reward or advantage is computed via a convex combination with coefficients $\{w_k\}_{k=1}^n$ and $\sum_k w_k = 1$. Unless otherwise specified, the coefficients corresponding to all rewards are equal. We implement our proposed DVAO and conduct all experiments based on verl \citep{verl} framework. For hyperparameters, we utilize the AdamW \citep{adamw} optimizer with a constant learning rate of $1 \times 10^{-6}$. For rollout, the prompt batch size is 128 and we sample $G=16$ responses for each prompt. For training, we train 500 steps to ensure convergence. The maximum number of tokens for generation is set to 8,192 tokens. We report avg@16 results for mathematical reasoning task. The inference hyperparameters of evaluation are set to temperature 0.6 and top-p 0.95. We conduct all experiments on a server with 8$\times$NVIDIA H20-3e GPUs and an Intel$^\text{®}$ Xeon$^\text{®}$ Platinum 8575C CPU.

\section{Limitations and Future Work}
\label{appendix:limitation}

While DVAO effectively addresses the limitations of fixed scalarization in multi-reward GRPO, there are a few boundaries to consider. First, the accuracy of DVAO's dynamic weighting relies on the empirical variance estimation within a rollout group ($G$). In our experiments, a standard group size of $G=16$ provided highly robust signals. However, for extremely large models where hardware memory constraints force very small group sizes (e.g., $G \le 4$), the intra-group variance estimation might become noisy. Future work could explore incorporating historical momentum or cross-batch moving averages to stabilize variance estimation under extreme memory constraints. Second, our empirical evaluations primarily focus on dual-objective scenarios (e.g., accuracy and length/format). Although our theoretical proofs mathematically hold for an arbitrary number of $n$ rewards, the empirical optimization dynamics in hyper-dimensional reward spaces—such as simultaneously aligning helpfulness, harmlessness, style, length, and tool-use—remain an open question for future exploration. Lastly, because DVAO inherently amplifies learning signals based on variance, its efficacy is tied to the quality of the underlying reward functions. If a poorly designed auxiliary reward exhibits artificially high variance due to noise rather than meaningful learning signals, DVAO may inadvertently up-weight it. Thus, while DVAO eliminates the need for manual weight tuning, it still operates optimally alongside reasonably well-calibrated individual reward definitions.



\end{document}